\documentclass[12pt]{article}
\usepackage{amsmath, amssymb, amsfonts}
\usepackage{amsthm}
\usepackage{geometry}
\usepackage{cite}
\usepackage[colorlinks=true,linkcolor=blue,citecolor=blue,urlcolor=blue]{hyperref}
\usepackage{tikz}
\usetikzlibrary{arrows.meta, positioning, calc}

\numberwithin{equation}{section}

\newtheorem{theorem}{Theorem}[section]

\newtheorem{corollary}{Corollary}[section]

\theoremstyle{definition}
\newtheorem{definition}{Definition}[section]
\newtheorem{remark}{Remark}[section]

\title{\textbf{Topological DeepONets and a generalization of the Chen--Chen operator approximation theorem}}

\author{\textsc{Vugar E. Ismailov}\thanks{The author can be contacted at
\texttt{vugaris@mail.ru} or \texttt{vugaris@gmail.com}.}}

\date{}

\begin{document}
\maketitle

\begin{abstract}
\noindent
Deep Operator Networks (DeepONets) provide a branch--trunk neural architecture
for approximating nonlinear operators acting between function spaces.
In the classical operator approximation framework, the input is a function
$u\in C(K_1)$ defined on a compact set $K_1$ (typically a compact subset of a
Banach space), and the operator maps $u$ to an output function
$G(u)\in C(K_2)$ defined on a compact Euclidean domain
$K_2\subset\mathbb{R}^d$.
In this paper, we develop a topological extension in which the operator input lies
in an arbitrary Hausdorff locally convex space $X$.
We construct topological feedforward neural networks on $X$ using continuous
linear functionals from the dual space $X^*$ and introduce
\emph{topological DeepONets} whose branch component acts on $X$ through such linear
measurements, while the trunk component acts on the Euclidean output domain.
Our main theorem shows that continuous operators
$G:V\to C(K;\mathbb{R}^m)$, where $V\subset X$ and $K\subset\mathbb{R}^d$ are compact,
can be uniformly approximated by such topological DeepONets.
This extends the classical Chen--Chen operator approximation theorem
from spaces of continuous functions to
locally convex spaces and yields a branch--trunk
approximation theorem beyond the Banach-space setting.
\end{abstract}



\section{Introduction}

Deep neural networks are usually used to approximate nonlinear mappings
between finite-dimensional Euclidean spaces. However, in many scientific and engineering
applications, the object of interest is not a function but an
operator: a mapping that takes an input function and returns an
output function. In the operator-learning viewpoint, one aims to learn
a continuous nonlinear operator
\[
G:V\longrightarrow C(K;\mathbb{R}^m),
\]
from sampled input--output data, where $V$ is a compact set of admissible inputs
and $K$ is the output domain.

A prominent architecture for this purpose is the \emph{Deep Operator Network}
(DeepONet), proposed in \cite{Lu}. In its standard form, DeepONet
employs a \emph{branch network} (encoding the input function through sensor
measurements) and a \emph{trunk network} (processing the variable $y$),
and combines their outputs through a dot product. In this way,
DeepONet approximates $G(u)(y)$ in a separable form as a finite sum of
functions of $y$, with coefficients depending on the input $u$.

A central theoretical motivation for DeepONets comes from the universal
approximation theorem for operators established by Chen and Chen \cite{Chen2}.
This theorem shows that continuous operators between spaces of continuous
functions can be uniformly approximated on compact sets by expressions depending
on finitely many point evaluations of the input function.
DeepONets place such
approximants into a structured branch--trunk architecture and, in practice, make
both sides deep.

Since its introduction, DeepONet has been successfully applied to a wide range
of engineering and scientific problems. Early studies demonstrated its ability to learn operators associated with
dynamical systems and diffusion–reaction equations \cite{Chen2,Lu}.
Subsequent works employed DeepONets for learning nonlinear operators arising
in multiphysics and multiscale models, including electro-convection phenomena
\cite{Cai2021}, hypersonic flows governed by the Navier–Stokes equations with
finite-rate chemistry \cite{Mao2021}, and multiscale bubble-growth dynamics
\cite{Lin2021}. Related developments include physics-informed DeepONets for
learning solution operators of parametric PDEs \cite{Wang2021}.

More recent developments extend DeepONet-based operator learning to a variety of
challenging application domains, including complex-valued formulations for
three-dimensional Maxwell equations \cite{Jiang2024}, surrogate modeling for
shape optimization \cite{Shukla2024shape}, learning two-phase microstructure
evolution using neural-operator and autoencoder architectures \cite{Oommen2022},
and aerothermodynamic analysis of hypersonic configurations \cite{Shukla2024aero}.
In addition, DeepONet architectures have been explored in control-oriented
settings, including neural-operator approximation of backstepping controller and observer
gain functions for reaction–diffusion PDEs \cite{Krstic2024}, predictor-based
stabilization of nonlinear systems with input delay \cite{Bhan2025}, and predictive control \cite{Jong2025}.

The aim of the present paper is to develop a framework in which the \emph{input}
of the operator need not lie in a Euclidean space or, more generally, in a normed
linear space (such as a Hilbert or Banach space). Instead, we treat the case where
the input belongs to a locally convex topological vector space $X$.
In this viewpoint, the network receives admissible measurements of the input element
in the form of continuous linear functionals. That is, each hidden neuron evaluates
a continuous linear functional on the input element and then applies the activation
function. This setting is natural when the input is an element of an abstract
function space endowed with a locally convex topology and the architecture is allowed
to use linear measurements compatible with that topology. Such situations arise
frequently in analysis and applications.

For instance, spaces of differentiable functions, which arise naturally in the
theory of partial differential equations, provide fundamental examples of
non-normable topological spaces. The Schwartz space $\mathcal{S}(\mathbb{R}^n)$ of rapidly decreasing smooth
functions is equipped with the countable family of seminorms
\[
p_{a,b}(f)=\sup_{x\in\mathbb{R}^n}\bigl|x^a D^b f(x)\bigr|,
\qquad a,b\in\mathbb{N}^n,
\]
where $x^a=x_1^{a_1}\cdots x_n^{a_n}$ and $D^b=\partial^{|b|}/\partial x_1^{b_1}\cdots\partial x_n^{b_n}$.
This space is complete and metrizable, hence a Fr\'echet space, but it is not normable.

Another important example is the space $\mathcal{D}(U)$ of smooth functions
with compact support in an open set $U\subset\mathbb{R}^n$. For each compact
set $K\subset U$, the subspace
\[
C_0^\infty(K)=\{f\in C^\infty(U):\operatorname{supp}(f)\subset K\}
\]
is a Fréchet space with seminorms
\[
p_{K,m}(f)=\max_{|\alpha|\le m}\sup_{x\in K}|D^\alpha f(x)|,
\qquad m\in\mathbb{N}.
\]
The space $\mathcal{D}(U)$ is obtained as the inductive limit of the spaces
$C_0^\infty(K)$, where $(K_j)$ is a directed family of compact subsets of $U$
whose union equals $U$, that is,
\[
\mathcal{D}(U)=\varinjlim C_0^\infty(K_j).
\]
In this topology, $\mathcal{D}(U)$ is locally convex and complete, but not metrizable and therefore not normable. The space
$\mathcal{D}(U)$, known as the space of test functions, plays a fundamental
role in the theory of distributions.

More generally, for a topological space $X$, the space $C(X)$ of continuous
functions endowed with the topology of uniform convergence on compact sets,
defined by the seminorms
\[
\varphi_K(f)=\max_{x\in K}|f(x)|,
\qquad K\subset X\ \text{compact},
\]
is a locally convex space that is not normable unless $X$ is
compact.

In our recent work \cite{Ism2026a} we proved a universal approximation theorem for
feedforward neural networks on topological vector spaces with the
Hahn--Banach extension property, in particular on locally convex spaces.
These networks are constructed using continuous linear functionals from the
dual space $X^*$ together with a fixed scalar activation function.
The present paper uses this density mechanism on compact sets to build
a DeepONet-type approximation framework and to prove an operator approximation
theorem of branch--trunk type.

Specifically, we introduce a topological DeepONet architecture in which the
branch component acts on a locally convex input space $X$ through continuous linear measurements
from $X^*$, while the trunk component acts on the output domain
$K\subset\mathbb{R}^d$. Within this setting we prove a universal approximation
theorem for continuous operators $G:V\to C(K;\mathbb{R}^m)$ on compact sets
$V\subset X$ by finite separable expansions whose coefficient maps are realized by topological neural networks on $X$.

As a consequence, we obtain approximation theorems related to DeepONets, showing
that continuous nonlinear operators admit approximations of branch--trunk type extending the
classical dot-product formulation beyond the Banach-space setting. The results
place the Chen--Chen operator approximation principle \cite{Chen2} and the DeepONet
architecture \cite{Lu} into a unified locally convex framework.

The remainder of the paper is organized as follows.
Section~2 introduces topological neural networks and recalls the
universality theorem on locally convex spaces.
Section~3 develops the topological DeepONet architecture and proves
the main operator approximation theorems and their corollaries.
Section~4 presents several examples.

\section{Topological networks on locally convex spaces}

This section provides definitions and results that will be used later.

Throughout the paper, $X$ denotes a locally convex topological vector space
and $X^*$ its continuous dual. All locally convex spaces are assumed to be
Hausdorff.

\begin{definition}[Topological neural network on a locally convex space]\label{def:1layer}
Fix an activation function $\sigma:\mathbb{R}\to\mathbb{R}$ and let $m\ge 1$.
A (vector-valued) \emph{topological feedforward neural network} on $X$
with one hidden layer is any mapping $H:X\to\mathbb{R}^m$ of the form
\[
H(x)=A\,\sigma\!\bigl(T(x)\bigr),
\]
where
\[
T(x)=(f_1(x)-\theta_1,\dots,f_r(x)-\theta_r),
\]
with $f_i\in X^*$, $\theta_i\in\mathbb{R}$, and $A\in\mathbb{R}^{m\times r}$.
Here $\sigma$ acts componentwise on $\mathbb{R}^r$. Each hidden neuron evaluates
a continuous linear functional on the input element and then applies the
activation function.
\end{definition}

Deep networks are obtained by alternating affine maps and componentwise activation functions, 
with the output layer given by a linear map. For completeness, we recall one convenient formulation.

\begin{definition}[Deep topological neural network on $X$]\label{def:Llayer}
Fix integers $L\ge 2$ and $m\ge 1$, and let $n_1,\dots,n_{L-1}\ge 1$.
A (vector-valued) \emph{deep topological feedforward neural network} on $X$ of depth $L$
is a mapping $H:X\to\mathbb{R}^m$ of the form
\[
H(x)=A_L\,z_{L-1}(x),
\]
where the hidden-layer outputs $z_\ell(x)\in\mathbb{R}^{n_\ell}$ are defined by
\[
z_1(x)=\sigma\!\bigl(T_1(x)\bigr),\qquad
z_{\ell}(x)=\sigma\!\bigl(A_\ell z_{\ell-1}(x)-b_\ell\bigr),\ \ \ell=2,\dots,L-1.
\]
Here $\sigma$ acts componentwise, each $A_\ell$ is a real matrix of appropriate size,
and each $b_\ell$ is a real bias vector. The first affine map $T_1:X\to\mathbb{R}^{n_1}$
is defined by
\[
T_1(x)=\bigl(f_1(x)-\theta_1,\dots,f_{n_1}(x)-\theta_{n_1}\bigr),
\]
with $f_j\in X^*$ and $\theta_j\in\mathbb{R}$.
\end{definition}

\begin{remark}
When $X=\mathbb{R}^d$ endowed with its usual topology, the above classes of
networks reduce to the standard feedforward neural networks used in the
classical theory.
Indeed, in this case every continuous linear functional $f\in X^*$ has the form
\[
f(x)=w\cdot x, \qquad x\in\mathbb{R}^d,
\]
for some vector $w\in\mathbb{R}^d$, where $w\cdot x$ denotes the Euclidean
inner product. This follows from the well known representation theorem for
continuous linear functionals on Hilbert spaces. Consequently, the expressions
$f_j(x)-\theta_j$ appearing in the first affine map become
\[
w_j\cdot x-\theta_j,
\]
which are exactly the affine forms used in classical neural networks on
$\mathbb{R}^d$. Therefore, the topological neural networks introduced above extend the
traditional Euclidean neural networks to inputs belonging to general
locally convex spaces.
\end{remark}

We work with the space $C(X;\mathbb{R}^m)$ of continuous functions from
$X$ into $\mathbb{R}^m$ equipped with
the topology of uniform convergence on compact sets. This topology is generated
by the family of seminorms
\[
\|g\|_{K}
=
\sup_{x\in K}\|g(x)\|_{\mathbb{R}^m},
\]
where $K$ ranges over all compact subsets of $X$ and
$\|\cdot\|_{\mathbb{R}^m}$ denotes any fixed norm on $\mathbb{R}^m$.
Since all norms on a finite-dimensional space are equivalent, the
resulting topology does not depend on the particular choice of the norm.

A subbasis at the origin for this topology is given by the sets
\[
U(K,r)
=
\left\{
g\in C(X;\mathbb{R}^m): \|g\|_{K}<r
\right\},
\]
where $K\subset X$ is compact and $r>0$.

Thus, when we say that a family of functions $\mathcal{F}$ acting from $X$
into $\mathbb{R}^m$ is dense in $C(X;\mathbb{R}^m)$, we mean density with
respect to the topology of uniform convergence on compact sets. That is,
for every compact $K\subset X$, every $g\in C(K;\mathbb{R}^m)$, and every
$\varepsilon>0$, there exists $f\in\mathcal{F}$ such that
\[
\|g-f\|_{K}<\varepsilon.
\]
Equivalently, for each compact $K\subset X$, the restrictions
\[
\{\,f|_{K} : f\in\mathcal{F}\,\}
\]
are dense in $C(K;\mathbb{R}^m)$ with respect to $\|\cdot\|_{K}$.

The space $C(X;\mathbb{R})$ will be denoted by $C(X)$.

\begin{definition}[Tauber--Wiener function]
A function $\sigma:\mathbb{R}\to\mathbb{R}$ is called a
\emph{Tauber--Wiener function} if the linear span of the set
\[
\{\sigma(wt-\theta): w,\theta\in\mathbb{R}\}
\]
is dense in $C([a,b])$ for every closed interval $[a,b]\subset\mathbb{R}$.
\end{definition}

Functions with this property generate dense families of translations and
dilations of $\sigma$ on every compact subset of the real line and play a
central role in neural network approximation theory. The terminology originates
from the work of Chen and Chen \cite{Chen2}, where this condition is used in
the study of operator approximation by neural networks. Tauber--Wiener
functions have also been exploited in several subsequent works
(see, e.g., \cite{Valle,Ism2026b}).

We denote by $\mathcal{S}_\sigma(X;\mathbb{R}^m)$ the class of all
single-hidden-layer topological feedforward neural networks $H:X\to\mathbb{R}^m$
constructed from continuous linear functionals in $X^*$ and the
activation function $\sigma$ (see Definition~\ref{def:1layer}).

A scalar-valued version of the following theorem was proved in
\cite{Ism2026a} for neural networks on topological vector spaces possessing
the Hahn--Banach extension property. Since every locally convex space has
this property, the scalar case follows from that result. In the present
paper we restrict attention to locally convex spaces, which are widely used
in analysis, although the results remain valid for topological vector spaces
with the Hahn--Banach extension property. For completeness, we include below
a direct proof for the vector-valued case.

\begin{theorem}\label{thm:univX}
Let $X$ be a locally convex topological vector space and assume that
the activation function $\sigma$ is a Tauber--Wiener function. Then for every
compact set $K\subset X$, every function $g\in C(K;\mathbb{R}^m)$, and every
$\varepsilon>0$, there exists a topological neural network
$H\in\mathcal{S}_\sigma(X;\mathbb{R}^m)$ such that
\[
\|g-H\|_{K}<\varepsilon.
\]
In other words, the class $\mathcal{S}_\sigma(X;\mathbb{R}^m)$ is dense in
$C(X;\mathbb{R}^m)$ with respect to the topology of uniform convergence on
compact subsets of $X$.
\end{theorem}

\begin{proof}
Fix a compact set $K\subset X$, a function
\[
g=(g_1,\dots,g_m)\in C(K;\mathbb{R}^m),
\]
and $\varepsilon>0$. We construct a topological neural network
$H\in\mathcal{S}_\sigma(X;\mathbb{R}^m)$ such that $\|g-H\|_K<\varepsilon$.

We equip $\mathbb{R}^m$ with the sup norm
\[
\|x\|_{\mathbb{R}^m}=\max_{1\le r\le m}|x_r|.
\]
Since all norms on the finite-dimensional space $\mathbb{R}^m$ are equivalent,
this choice does not change the topology of uniform convergence on compact sets.

Then
\[
\|g-H\|_{K}
=
\sup_{x\in K}\|g(x)-H(x)\|_{\mathbb{R}^m}
=
\sup_{x\in K}\max_{1\le r\le m}|g_r(x)-H_r(x)|.
\]
Hence
\[
\|g-H\|_{K}
=
\max_{1\le r\le m}\|g_r-H_r\|_{K},
\]
where for scalar-valued functions on $K$ we write
$\|h\|_K=\sup_{x\in K}|h(x)|$.

Thus it suffices to construct, for each $r=1,\dots,m$, a scalar topological neural network $H_r$
such that
\[
\|g_r-H_r\|_K<\varepsilon.
\]

Fix $r\in\{1,\dots,m\}$. Consider the set
\[
\mathcal{E}
=
\operatorname{span}\{\,e^{\ell(x)}:\ \ell\in X^*\,\}
\subset C(K).
\]
This set is an algebra, since for $\ell_1,\ell_2\in X^*$ we have
\[
e^{\ell_1(x)}e^{\ell_2(x)} = e^{(\ell_1+\ell_2)(x)},
\qquad \ell_1+\ell_2\in X^*,
\]
and it contains the constant functions because $e^{0}=1$.

Moreover, $\mathcal{E}$ separates points of $K$. Indeed, if $x,y\in K$
with $x\neq y$, then by the Hahn--Banach continuous extension theorem there exists
$\ell\in X^*$ such that $\ell(x)\neq \ell(y)$ (see, e.g.,
\cite[Theorem~3.6]{Rudin}). Consequently,
\[
e^{\ell(x)}\neq e^{\ell(y)}.
\]
Therefore, by the Stone--Weierstrass theorem, $\mathcal{E}$ is dense in
$C(K)$ with respect to the uniform norm.

Hence there exist $\ell_1,\dots,\ell_M\in X^*$ and real coefficients
$\alpha_1,\dots,\alpha_M$ such that
\begin{equation}\label{eq:expapprox}
\left\|
g_r-\sum_{i=1}^M \alpha_i e^{\ell_i(\cdot)}
\right\|_{K}
<
\frac{\varepsilon}{2}.
\end{equation}

Since each $\ell_i$ is continuous and $K$ is compact, the image
$\ell_i(K)$ is a compact subset of $\mathbb{R}$. Hence there exists a
compact interval $[a_i,b_i]$ such that $\ell_i(K)\subset[a_i,b_i]$.

Because $\sigma$ is a Tauber--Wiener function, the linear span of the
family $\{\sigma(wt-\theta):w,\theta\in\mathbb{R}\}$ is dense in
$C([a_i,b_i])$. Applying this property to the function $t\mapsto e^t$,
we obtain, for each $i$, an integer $N_i$ and real parameters
$c_{i,j},w_{i,j},\theta_{i,j}$ such that
\[
\sup_{t\in[a_i,b_i]}
\left|
e^t-\sum_{j=1}^{N_i} c_{i,j}\sigma(w_{i,j}t-\theta_{i,j})
\right|
<
\frac{\varepsilon}{2(1+\sum_{i=1}^M|\alpha_i|)}.
\]

Since $\ell_i(K)\subset[a_i,b_i]$, substituting $t=\ell_i(x)$ gives
\begin{equation}\label{eq:twapprox}
\sup_{x\in K}
\left|
e^{\ell_i(x)}
-
\sum_{j=1}^{N_i} c_{i,j}\sigma(w_{i,j}\ell_i(x)-\theta_{i,j})
\right|
<
\frac{\varepsilon}{2(1+\sum_{i=1}^M|\alpha_i|)}.
\end{equation}

Note that for each $i,j$, the map $x\mapsto w_{i,j}\ell_i(x)$ is again a
continuous linear functional on $X$, that is,
$w_{i,j}\ell_i\in X^*$.
Thus each term $\sigma(w_{i,j}\ell_i(x)-\theta_{i,j})$ has the form
$\sigma(f(x)-\theta)$ with $f\in X^*$.

Define
\[
H_r(x)
=
\sum_{i=1}^M \alpha_i
\sum_{j=1}^{N_i}
c_{i,j}\sigma\bigl(w_{i,j}\ell_i(x)-\theta_{i,j}\bigr),
\qquad x\in X.
\]
Then $H_r\in\mathcal{S}_\sigma(X;\mathbb{R})$.

Using \eqref{eq:twapprox}, multiplying the corresponding estimates by
$|\alpha_i|$ and summing over $i=1,\dots,M$, we obtain
\[
\left\|
\sum_{i=1}^M \alpha_i e^{\ell_i(\cdot)} - H_r
\right\|_{K}
<
\frac{\varepsilon}{2}.
\]
Combining this with \eqref{eq:expapprox} yields
\[
\|g_r-H_r\|_K
<
\varepsilon.
\]

Repeating the construction for each component $r=1,\dots,m$, we obtain
scalar topological neural networks $H_1,\dots,H_m$. Define
\[
H(x)=(H_1(x),\dots,H_m(x)),\qquad x\in X.
\]
Then $H\in\mathcal{S}_\sigma(X;\mathbb{R}^m)$ and
\[
\|g-H\|_K
=
\sup_{x\in K}\max_{1\le r\le m}|g_r(x)-H_r(x)|
<
\varepsilon.
\]
This completes the proof.
\end{proof}

In the next section, topological neural networks on $X$ play the role of the branch component in a
DeepONet. Theorem~\ref{thm:univX} provides the tool for approximating continuous
coefficient maps in the locally convex setting, where the
available measurements are given by continuous linear functionals.

The preceding result provides a universality principle for neural networks on locally
convex input spaces. An important example arises when the input space is the Banach
space $C(K)$ endowed with the uniform norm. In this case, Theorem~\ref{thm:univX} 
yields the classical operator approximation
theorem of Chen and Chen \cite{Chen2}.

\begin{theorem}[Chen and Chen \cite{Chen2}]
\label{thm:chen-chen}
Let $K$ be a compact metric space and let $V\subset C(K)$ be compact.
Assume that the activation $\sigma \in C(\mathbb{R})$ is a
Tauber--Wiener function. Then for every continuous functional $f\in C(V)$
and every $\varepsilon>0$ there exist integers $N,k\ge1$, points
$x_1,\dots,x_k\in K$, and real parameters $c_i,\theta_i,\xi_{ij}$
($i=1,\dots,N$, $j=1,\dots,k$) such that
\[
\left|
f(u)-\sum_{i=1}^{N}c_i\,
\sigma\!\left(\sum_{j=1}^{k}\xi_{ij}\,u(x_j)-\theta_i\right)
\right|
<\varepsilon
\]
holds for all $u\in V$.
\end{theorem}

\begin{proof}
We regard $C(K)$ as a Banach space and hence as a locally convex space.
Thus, Theorem~\ref{thm:univX} applies with $X=C(K)$, compact set $V\subset X$,
and $m=1$. Therefore, there exist $N\ge1$, continuous linear functionals
$\ell_1,\dots,\ell_N\in C(K)^*$, and real numbers $c_i,\theta_i$ such that
\begin{equation}
\sup_{u\in V}\left|
f(u)-\sum_{i=1}^{N} c_i\,\sigma(\ell_i(u)-\theta_i)
\right|
<\frac{\varepsilon}{2}.
\label{eq:sA}
\end{equation}

Each $\ell_i$ is a continuous linear functional on $C(K)$ and hence, by the
Riesz representation theorem, admits a representation
\[
\ell_i(u)=\int_K u\,d\mu_i
\]
for a finite signed Borel measure $\mu_i$ on $K$.

Since $V\subset C(K)$ is compact, it is equicontinuous by the
Arzel\`a--Ascoli theorem. Consequently, for every $\delta>0$, each functional $\ell_i$ can be
uniformly approximated on $V$ by a Riemann-type sum for the integral
$\int_K u\,d\mu_i$, i.e., by a finite linear combination of
point evaluations. That is, there exist points
$x_1,\dots,x_k\in K$ and coefficients $\xi_{ij}$ such that
\[
\sup_{u\in V}\left|
\ell_i(u)-\sum_{j=1}^{k}\xi_{ij}\,u(x_j)
\right|
<\delta,
\qquad i=1,\dots,N.
\]

Since $\sigma$ is uniformly continuous on compact intervals containing the
ranges of the arguments, we may choose $\delta$ sufficiently small so that
\begin{equation}
\sup_{u\in V}\left|
\sum_{i=1}^{N} c_i\,\sigma(\ell_i(u)-\theta_i)
-
\sum_{i=1}^{N} c_i\,\sigma\!\left(\sum_{j=1}^{k}\xi_{ij}u(x_j)-\theta_i\right)
\right|
<
\frac{\varepsilon}{2}.
\label{eq:sC}
\end{equation}

Combining \eqref{eq:sA} and \eqref{eq:sC} by the triangle inequality
gives the desired $\varepsilon$-approximation.
\end{proof}

\section{DeepONets on locally convex spaces and operator approximation}

Let $X$ be a locally convex topological vector space, $V\subset X$ a compact set, and
$K\subset\mathbb{R}^d$ a compact set. We consider continuous operators
\[
G:V\longrightarrow C(K;\mathbb{R}^m),
\]
where $C(K;\mathbb{R}^m)$ is equipped with the uniform norm $\|\cdot\|_K$.
In the DeepONet framework, one seeks to approximate the mapping
\[
(u,y)\longmapsto G(u)(y), \qquad (u,y)\in V\times K,
\]
by a structured network that separates the dependence on the input $u$ from
the dependence on the coordinate $y$.

\begin{definition}[Topological DeepONet]\label{def:TDeepONet}
Fix integers $p\ge 1$ and $m\ge 1$. A \emph{topological DeepONet} consists of

\begin{itemize}

\item[1)] a \emph{branch network}
\[
\mathcal{B}:X\to\mathbb{R}^{m\times p},
\]
which takes an input $u\in X$ and outputs a matrix
\[
\mathcal{B}(u)=[b_1(u)\ \cdots\ b_p(u)],
\]
where each column $b_k:X\to\mathbb{R}^m$ is a topological neural network on $X$
constructed using continuous linear functionals from the dual space $X^*$
and an activation function $\sigma$ (see Definition~\ref{def:1layer});

\item[2)] a \emph{trunk network}
\[
\mathcal{T}:\mathbb{R}^d\to\mathbb{R}^{p},
\]
which takes a point $y\in\mathbb{R}^d$ as input and outputs the vector
\[
\mathcal{T}(y)=(t_1(y),\dots,t_p(y))^T,
\]
represented by a Euclidean neural network (for example, a single-hidden-layer network),
whose components $t_k:\mathbb{R}^d\to\mathbb{R}$, $k=1,\dots,p$, denote
the $p$ outputs of the network.
\end{itemize}

The resulting DeepONet defines an operator
\[
\widehat G:X\to C(\mathbb{R}^d,\mathbb{R}^m)
\]
given by
\[
(\widehat G(u))(y)=\mathcal{B}(u)\,\mathcal{T}(y),
\qquad u\in X,\; y\in\mathbb{R}^d.
\]

Equivalently,
\begin{equation}\label{eq:separableSum}
\widehat G(u)(y)=\sum_{k=1}^{p} b_k(u)\,t_k(y).
\end{equation}
\end{definition}

The architecture of the topological DeepONet is illustrated in Figure~\ref{fig:topodeeponet}.

\begin{figure}[!htbp]
\centering
\begin{tikzpicture}[
node distance=1.9cm,
every node/.style={draw, rectangle, rounded corners, minimum width=3cm, minimum height=1cm, align=center},
arrow/.style={->, thick}
]

\node (u) {$u\in X$ \\ \small(locally convex space)};

\node (sensors) [below=of u]
{$f_1(u),\,f_2(u),\,\dots,\,f_r(u)$ \\ \small$f_j\in X^*$};

\node (branch) [below=of sensors]
{Branch network};

\node (b) [below=of branch]
{$[b_1(u),\,\dots,\,b_p(u)]$};

\node (y) [right=6cm of sensors] {$y\in\mathbb{R}^d$};

\node (trunk) [below=of y]
{Trunk network};

\node (t) [below=of trunk]
{$[t_1(y),\,\dots,\,t_p(y)]^T$};

\node (combine) [draw, circle, minimum size=1.3cm] at ($(b)!0.5!(t)$) {$\times$};

\node (out) [below=2cm of combine]
{$\widehat G(u)(y)=\displaystyle\sum_{k=1}^{p} b_k(u)t_k(y)$};

\draw[arrow] (u) -- (sensors);
\draw[arrow] (sensors) -- (branch);
\draw[arrow] (branch) -- (b);

\draw[arrow] (y) -- (trunk);
\draw[arrow] (trunk) -- (t);

\draw[arrow] (b) -- (combine);
\draw[arrow] (t) -- (combine);

\draw[arrow] (combine) -- (out);

\end{tikzpicture}
\caption{Topological DeepONet architecture. The branch network encodes the input 
element $u\in X$, where $X$ is a locally convex space, through finitely many linear 
measurements $f_1(u),\dots,f_r(u)$ with $f_j\in X^*$. The trunk network takes $y\in\mathbb{R}^d$ 
as input and yields $[t_1(y),\dots,t_p(y)]^T\in\mathbb{R}^p$. The outputs of the branch and 
trunk networks are then multiplied to produce the final output. If $X$ is the space of 
continuous functions and the functionals $f_j$ are point evaluation functionals, 
$f_j(u)=u(x_j)$, then the classical DeepONet architecture is recovered.}
\label{fig:topodeeponet}
\end{figure}
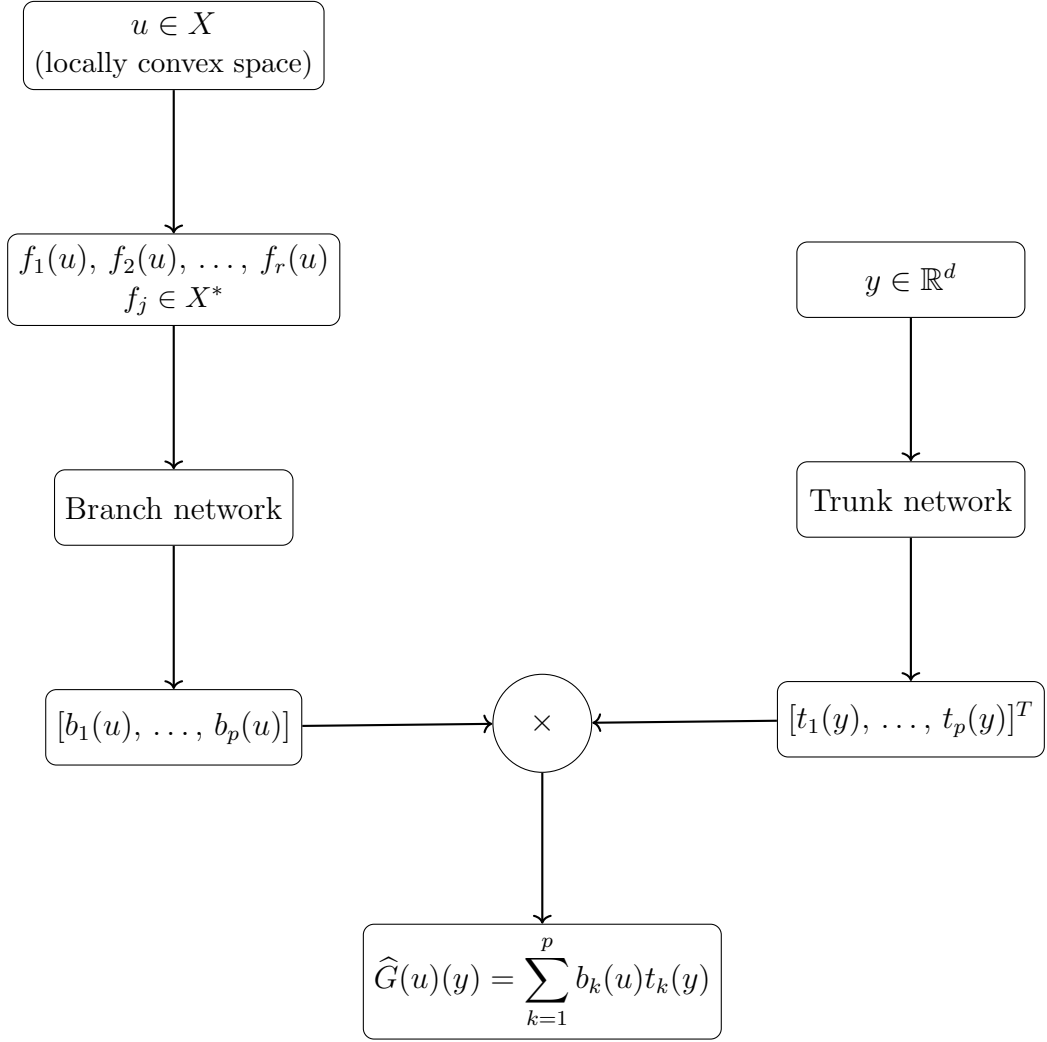

\begin{remark}
When $m=1$, \eqref{eq:separableSum} reduces to the classical dot-product form
\[
\widehat G(u)(y)=\langle b(u),t(y)\rangle,
\]
with $b(u)\in\mathbb{R}^p$ and $t(y)\in\mathbb{R}^p$. This is the standard
DeepONet architecture introduced in \cite{Lu}. The above definition
extends this construction to inputs from general locally convex spaces and to
vector-valued outputs, while preserving the explicit branch--trunk separation.
\end{remark}

For example, when $m=1$ and the trunk network is realized by a
single-hidden-layer neural network, the representation \eqref{eq:separableSum}
can be written explicitly in terms of activation functions. In particular,
if the branch network uses linear measurements $f_j(u)$, it takes the form
\[
\widehat G(u)(y)
=
\sum_{k=1}^{p}
\underbrace{\sum_{i=1}^{n}
c_{ki}\,
\sigma\!\left(\sum_{j=1}^{r}\xi_{kij}f_j(u)-\theta_{ki}\right)}_{\text{branch}}
\;
\underbrace{\sigma(\omega_k\cdot y+\zeta_k)}_{\text{trunk}}.
\]
This coincides with the displayed formula in \cite{Lu} when
$X$ is the space of continuous functions and the functionals $f_j(u)$ are taken as point evaluations
$f_j(u)=u(x_j)$.

Theorem~\ref{thm:univX} provides approximation capability on the branch network side.
On the trunk network side, we rely on the classical density of 
ridge networks (i.e., single-hidden-layer neural networks) on
compact subsets of $\mathbb{R}^d$. It is well known that if the activation
$\sigma$ is a Tauber--Wiener function, then finite linear combinations of ridge functions
\[
y \longmapsto \sigma(\omega\cdot y+\zeta),
\qquad y,\omega\in\mathbb{R}^d,\ \zeta\in\mathbb{R},
\]
are dense in $C(K)$ for every compact $K\subset\mathbb{R}^d$
(see, e.g., \cite{Pinkus,Leshno}). For background on ridge functions, see
\cite{Ism,Pin}.

We now prove a universal approximation theorem for continuous operators
$G:V\to C(K;\mathbb{R}^m)$ by finite separable expansions of the form
\eqref{eq:separableSum}, where the coefficient maps are realized by branch
topological neural networks on $X$.

\begin{theorem}
\label{thm:operator-separable}
Let $X$ be a locally convex topological vector space and let $V\subset X$ be compact.
Let $K\subset\mathbb{R}^d$ be compact and let
$G:V\to C(K;\mathbb{R}^m)$ be continuous.
Assume that the activation $\sigma\in C(\mathbb{R})$
is a Tauber--Wiener function.

Then for every $\varepsilon>0$ there exist an integer $N\ge1$, ridge functions
\[
\phi_k(y)=\sigma(\omega_k\cdot y+\zeta_k),
\qquad k=1,\dots,N,
\]
and topological neural networks
\[
a_k:X\to\mathbb{R}^m,
\qquad k=1,\dots,N,
\]
constructed according to Definition~\ref{def:1layer},
such that
\[
\sup_{u\in V}\ \sup_{y\in K}
\left\|
G(u)(y)-\sum_{k=1}^N a_k(u)\,\phi_k(y)
\right\|_{\mathbb{R}^m}
<\varepsilon.
\]
\end{theorem}

\begin{proof}
Define
\[
W := G(V) \subset C(K;\mathbb{R}^m).
\]
Since $V$ is compact and $G$ is continuous, $W$ is compact in
$C(K;\mathbb{R}^m)$ equipped with the uniform norm
\[
\|h\|_{K}=\sup_{y\in K}\|h(y)\|_{\mathbb{R}^m}.
\]
We equip $\mathbb{R}^m$ with the sup-norm
\[
\|x\|_{\mathbb{R}^m}=\max_{1\le r\le m}|x_r|.
\]
Since $\mathbb{R}^m$ is finite-dimensional, all norms on $\mathbb{R}^m$ are
equivalent; hence this choice does not affect compactness or approximation
properties, and simplifies componentwise estimates below.

Fix an arbitrary $\varepsilon>0$ and set $\delta:=\varepsilon/4$.

Let $h=(h_1,\dots,h_m)\in W$. For each component $r=1,\dots,m$,
the density of ridge networks on $K$ yields a finite ridge expansion
\[
R_{h,r}(y)
=
\sum_{j=1}^{N(h,r)} \alpha_{h,r,j}\,
\sigma(\omega_{h,r,j}\cdot y+\zeta_{h,r,j}),
\qquad y\in K,
\]
such that $\|h_r-R_{h,r}\|_K<\delta$.

Define the vector-valued function
\[
R_h(y):=(R_{h,1}(y),\dots,R_{h,m}(y)),
\qquad y\in K.
\]
Then
\[
\|h-R_h\|_K<\delta.
\]

Define an open neighborhood of $h$ in $W$ by
\[
U_h:=\{h'\in W:\ \|h'-h\|_K<\delta\}.
\]
Then for every $h'\in U_h$,
\begin{equation}\label{eq:triangle}
\|h'-R_h\|_K
\le
\|h'-h\|_K+\|h-R_h\|_K
<
\delta+\delta
=
2\delta.
\end{equation}

Since $W$ is compact, the open cover $\{U_h:h\in W\}$ admits a finite
subcover. Hence there exist $h^{(1)},\dots,h^{(M)}\in W$ such that
\begin{equation}\label{eq:finite-cover}
W\subset \bigcup_{j=1}^M U_{h^{(j)}}.
\end{equation}

Now collect all ridge functions appearing in the finitely many approximants
$R_{h^{(j)}}$. Denote the elements of this finite set by
\[
\phi_k(y)=\sigma(\omega_k\cdot y+\zeta_k),
\qquad k=1,\dots,N.
\]
By adding zero coefficients if necessary, for each $j=1,\dots,M$ we may
write
\begin{equation}\label{eq:Rj}
R_{h^{(j)}}(y)
=
\sum_{k=1}^N A_{j,k}\,\phi_k(y),
\qquad y\in K,
\end{equation}
with vectors $A_{j,k}\in\mathbb{R}^m$.

Since $W$ is a compact metric space, it is paracompact.
By the partition of unity theorem (see Dugundji \cite[p.~170, Theorem~4.2]{Dugun}),
the finite open cover \eqref{eq:finite-cover} admits a continuous partition
of unity subordinate to it. Hence there exist continuous functions
\[
\eta_j:W\to[0,1],\qquad j=1,\dots,M,
\]
such that
\begin{itemize}
\item [1)] $\sum_{j=1}^M \eta_j(h)=1$ for all $h\in W$;
\item [2)] $\operatorname{supp}(\eta_j)\subset U_{h^{(j)}}$.
\end{itemize}

Define continuous coefficient maps $c_k:W\to\mathbb{R}^m$ by
\[
c_k(h):=\sum_{j=1}^M \eta_j(h)\,A_{j,k},
\qquad k=1,\dots,N.
\]
Define
\[
A:W\to C(K;\mathbb{R}^m),
\qquad
(A(h))(y):=\sum_{k=1}^N c_k(h)\,\phi_k(y).
\]

We claim that
\begin{equation}\label{eq:hA}
\|h-A(h)\|_K\le 2\delta
\quad\text{for all }h\in W.
\end{equation}

Indeed, using \eqref{eq:Rj} and the definition of $c_k$,
\[
A(h)(y)=\sum_{j=1}^M \eta_j(h)\,R_{h^{(j)}}(y),
\]
hence
\[
h(y)-A(h)(y)
=
\sum_{j=1}^M \eta_j(h)\bigl(h(y)-R_{h^{(j)}}(y)\bigr).
\]
If $\eta_j(h)\neq0$, then $h\in U_{h^{(j)}}$, and therefore by \eqref{eq:triangle}
we have $\|h-R_{h^{(j)}}\|_K<2\delta$. Thus
\[
\|h(y)-A(h)(y)\|_{\mathbb{R}^m}
\le
\sum_{j=1}^M \eta_j(h)\cdot 2\delta
=
2\delta.
\]
Taking the maximum over $y\in K$ yields \eqref{eq:hA}.

For each $k$ and $r=1,\dots,m$, define
\[
b_{k,r}(u):=\bigl(c_k(G(u))\bigr)_r,
\qquad u\in V.
\]
Each $b_{k,r}$ is continuous on $V$.

Set
\[
M_k:=\max_{y\in K}|\phi_k(y)|<\infty.
\]
By Theorem~\ref{thm:univX}, for each $(k,r)$ there exists a scalar topological network
$a_{k,r}:X\to\mathbb{R}$ such that
\begin{equation}\label{eq:akr}
\sup_{u\in V}|b_{k,r}(u)-a_{k,r}(u)|
<
\frac{\delta}{N(M_k+1)m}.
\end{equation}

Define
\[
a_k(u)=(a_{k,1}(u),\dots,a_{k,m}(u)), \qquad k=1,\dots,N,
\]
and set
\[
\widetilde G(u)(y)
:=
\sum_{k=1}^N a_k(u)\,\phi_k(y).
\]

Fix $u\in V$ and $y\in K$. By \eqref{eq:hA},
\[
\|G(u)(y)-A(G(u))(y)\|_{\mathbb{R}^m}\le 2\delta.
\]
Moreover,
\[
A(G(u))(y)-\widetilde G(u)(y)
=
\sum_{k=1}^N (c_k(G(u))-a_k(u))\,\phi_k(y).
\]
Using \eqref{eq:akr} and $|\phi_k(y)|\le M_k$, we obtain
\[
\|A(G(u))(y)-\widetilde G(u)(y)\|_{\mathbb{R}^m}
\le
\sum_{k=1}^N \frac{\delta}{N(M_k+1)m}\cdot M_k
\le
\delta.
\]

Hence
\[
\|G(u)(y)-\widetilde G(u)(y)\|_{\mathbb{R}^m}
\le
2\delta+\delta
<
\varepsilon.
\]
Taking supremum over $u\in V$ and $y\in K$
completes the proof.
\end{proof}

We now state an approximation theorem in the style of the dot-product DeepONet
formulation (compare \cite[Theorem~2]{Lu}).

\begin{theorem}\label{thm:TopDeepONetTheorem2}
Under the assumptions of Theorem~\ref{thm:operator-separable}, for every
$\varepsilon>0$ there exists an integer $p\ge 1$, ridge functions
$t_k(y)=\sigma(\omega_k\cdot y+\zeta_k)$ on $K\subset\mathbb{R}^d$, and a trunk map
\[
\mathcal{T}(y)=\bigl(t_1(y),\dots,t_p(y)\bigr)^T\in\mathbb{R}^{p},
\qquad y\in K,
\]
together with a branch map $\mathcal{B}:X\to\mathbb{R}^{m\times p}$, whose
columns are topological neural networks on $X$
constructed as in Definition~\ref{def:1layer},
such that
\[
\sup_{u\in V}\ \sup_{y\in K}
\|G(u)(y)-\mathcal{B}(u)\mathcal{T}(y)\|_{\mathbb{R}^m}<\varepsilon.
\]
In other words, the operator $G$ admits an $\varepsilon$--approximation on
$V\times K$ by a DeepONet in the sense of
Definition~\ref{def:TDeepONet}.
\end{theorem}

\begin{proof}
By Theorem~\ref{thm:operator-separable}, there exists a separable expansion
\[
G(u)(y)\approx \sum_{k=1}^{p} a_k(u)\,t_k(y),
\qquad t_k(y)=\sigma(\omega_k\cdot y+\zeta_k).
\]
Define $\mathcal{T}(y)=(t_1(y),\dots,t_p(y))^T$ and define $\mathcal{B}(u)$ as
the $m\times p$ matrix whose $k$th column is $a_k(u)$. Then
\[
\sum_{k=1}^{p} a_k(u)\,t_k(y)=\mathcal{B}(u)\mathcal{T}(y),
\]
and the claimed uniform error bound is exactly the estimate obtained in
Theorem~\ref{thm:operator-separable}.
\end{proof}

\begin{remark}
When $m=1$, the matrix--vector product reduces to a dot product:
$\mathcal{B}(u)\mathcal{T}(y)=\langle b(u),t(y)\rangle$, which is the dot-product
formulation of DeepONets in \cite{Lu}. Theorem~\ref{thm:TopDeepONetTheorem2}
therefore extends the DeepONet approximation theorem to operators whose input
lies in a compact subset of a locally convex space and whose output functions are $\mathbb{R}^m$-valued.
\end{remark}

\begin{remark}
In classical DeepONets, the branch network typically receives measurements of
the input function at finitely many sensor locations, that is, values of the form
$u(x_1),\dots,u(x_r)$. Such measurements are meaningful because the input belongs
to a function space, where pointwise evaluation is naturally defined.

In the present setting, the input $u$ is an element of a locally convex space $X$
and need not be a function. Consequently, pointwise sampling is not available in
general. Instead, the branch encoder accesses $u$ through finitely many continuous
linear functionals $\ell(u)$, where $\ell\in X^*$. These functionals play the role
of generalized sensors: each network uses finitely many of them to form a finite
measurement vector.

When $X$ happens to be a function space, classical sensors are recovered as a
special case. For example, if $X=C(\Omega)$ is the space of continuous functions on a compact set $\Omega$
endowed with the uniform norm topology,
then the point evaluation maps $u\mapsto u(x)$ are continuous linear functionals,
hence belong to $X^*$. Thus continuous linear functionals provide a natural and
flexible abstract measurement interface for operator learning in the locally
convex setting.
\end{remark}

\begin{remark}
The trunk side in Theorem~\ref{thm:TopDeepONetTheorem2} is presented in ridge
form because it is one of the simplest universal approximators on compacts of $\mathbb{R}^d$.
In applications, one may replace the ridge trunk by a deep neural network
(ResNet, CNN, etc.) provided that the chosen trunk class is universal on $K$.
Likewise, on the branch side one may choose shallow or deep topological networks on $X$
constructed as in Definitions~\ref{def:1layer} and~\ref{def:Llayer}. The
approximation mechanism requires only density on compact sets, as captured
abstractly by Theorem~\ref{thm:univX}.
\end{remark}

The preceding remarks clarify the interpretation of the branch--trunk
construction and its relation to classical DeepONets. We now record two
direct consequences of Theorem~\ref{thm:operator-separable} showing that,
when the input space is a continuous function space and the admissible
measurements are chosen accordingly, one recovers the operator-approximation
results that form the theoretical foundation of DeepONets.

In particular, the classical Chen--Chen operator approximation theorem and the dot-product DeepONet
approximation theorem of Lu et al.\ appear as special cases of the present
locally convex framework.

\begin{corollary}[see Theorem 5 in \cite{Chen2} or Theorem 1 in \cite{Lu}]
Suppose $E$ is a Banach
space, $K_1\subset E$ and $K_2\subset\mathbb{R}^d$ are compact sets in $E$ and
$\mathbb{R}^d$, respectively, $V\subset C(K_1)$ is compact, and
$G:V\to C(K_2)$ is a continuous nonlinear operator.
Assume that the activation $\sigma\in C(\mathbb{R})$ is a Tauber--Wiener function.

Then for every $\varepsilon>0$ there exist integers
$n,p,r\ge1$, points $x_1,\dots,x_r\in K_1$, parameters
$c_{ki},\theta_{ki},\xi_{kij},\zeta_k\in\mathbb{R}$,
and vectors $\omega_k\in\mathbb{R}^d$ such that
\[
\left|
G(u)(y)-
\sum_{k=1}^p\sum_{i=1}^n
c_{ki}\,
\sigma\!\left(\sum_{j=1}^r \xi_{kij}u(x_j)-\theta_{ki}\right)
\sigma(\omega_k\cdot y+\zeta_k)
\right|
<\varepsilon
\]
for all $u\in V$ and $y\in K_2$.
\end{corollary}

\begin{proof}
Apply Theorem~\ref{thm:operator-separable} with $X=C(K_1)$, compact
$V\subset C(K_1)$, $K=K_2$, and $m=1$. Then there exist $p\ge1$, ridge functions
$\phi_k(y)=\sigma(\omega_k\cdot y+\zeta_k)$, and topological neural networks
\[
a_k:C(K_1)\to\mathbb{R},
\qquad k=1,\dots,p,
\]
such that
\begin{equation}\label{eq:cor-short-1}
\sup_{u\in V}\sup_{y\in K_2}
\left|
G(u)(y)-\sum_{k=1}^{p} a_k(u)\,\phi_k(y)
\right|<\frac{\varepsilon}{2}.
\end{equation}

Set
\[
M:=\max_{1\le k\le p}\ \sup_{y\in K_2}|\phi_k(y)|<\infty,
\qquad
\eta:=\frac{\varepsilon}{2p(M+1)}.
\]
For each $k$, apply Theorem~\ref{thm:chen-chen} to the continuous functional
$a_k|_V:V\to\mathbb{R}$
to obtain an approximation of the form
\[
\widetilde a_k(u)=\sum_{i=1}^{n}
c_{ki}\,
\sigma\!\left(\sum_{j=1}^{r}\xi_{kij}u(x_j)-\theta_{ki}\right)
\]
(with a common choice of $r,n$ after unifying finitely many sensor points and
adding zero coefficients if necessary) satisfying
\begin{equation}\label{eq:cor-short-2}
\sup_{u\in V}|a_k(u)-\widetilde a_k(u)|<\eta,
\qquad k=1,\dots,p.
\end{equation}
Define
\[
\widetilde G(u)(y):=\sum_{k=1}^{p}\widetilde a_k(u)\,\phi_k(y)
=\sum_{k=1}^{p}\sum_{i=1}^{n}
c_{ki}\,
\sigma\!\left(\sum_{j=1}^{r}\xi_{kij}u(x_j)-\theta_{ki}\right)
\sigma(\omega_k\cdot y+\zeta_k).
\]
Then for $u\in V$ and $y\in K_2$, by \eqref{eq:cor-short-1}--\eqref{eq:cor-short-2},
\begin{align*}
|G(u)(y)-\widetilde G(u)(y)|
&\le
\left|G(u)(y)-\sum_{k=1}^{p} a_k(u)\phi_k(y)\right|
+
\sum_{k=1}^{p}|a_k(u)-\widetilde a_k(u)|\,|\phi_k(y)| \\
&<
\frac{\varepsilon}{2}+p\eta M
\le
\frac{\varepsilon}{2}+p\eta(M+1)
=
\varepsilon .
\end{align*}
Taking the supremum over $u\in V$ and $y\in K_2$ yields the required uniform
estimate.
\end{proof}

\begin{corollary}[see Theorem 2 in \cite{Lu}]
Let $E$ be a Banach space, $K_1\subset E$ and $K_2\subset\mathbb{R}^d$
compact sets, and let $V\subset C(K_1)$ be compact.
Assume that $G:V\to C(K_2)$ is a continuous nonlinear operator and that the
activation $\sigma\in C(\mathbb{R})$ is a Tauber--Wiener function.

Then for every $\varepsilon>0$ there exist integers $r,p\ge1$, points
$x_1,\dots,x_r\in K_1$, and continuous mappings
\[
\mathcal{B}:\mathbb{R}^r\to\mathbb{R}^p,
\qquad
\mathcal{T}:\mathbb{R}^d\to\mathbb{R}^p,
\]
such that
\[
\left|
G(u)(y)-
\big\langle
\mathcal{B}\big(u(x_1),\dots,u(x_r)\big),\, \mathcal{T}(y)
\big\rangle
\right|
<\varepsilon
\]
for all $u\in V$ and $y\in K_2$.

Moreover, the maps $\mathcal{B}$ and $\mathcal{T}$ may be chosen from any class
of neural networks on $\mathbb{R}^r$ and $\mathbb{R}^d$, respectively, that is
dense in the corresponding spaces of continuous functions on compact sets
(e.g. fully connected, residual, or convolutional architectures).
\end{corollary}

\begin{proof}
Fix $\varepsilon>0$. By the preceding corollary, there exist integers
$r,n,p\ge1$, points $x_1,\dots,x_r\in K_1$, real parameters
$c_{ki},\theta_{ki},\xi_{kij},\zeta_k$, and vectors $\omega_k\in\mathbb{R}^d$
such that
\begin{equation}\label{eq:Lu-cor-prev}
\sup_{u\in V}\sup_{y\in K_2}
\left|
G(u)(y)-
\sum_{k=1}^p\sum_{i=1}^n
c_{ki}\,
\sigma\!\left(\sum_{j=1}^r \xi_{kij}u(x_j)-\theta_{ki}\right)
\sigma(\omega_k\cdot y+\zeta_k)
\right|
<\varepsilon.
\end{equation}

Define $\mathcal{B}:\mathbb{R}^r\to\mathbb{R}^p$ by
\[
\mathcal{B}(z_1,\dots,z_r)
:=
\Bigg(
\sum_{i=1}^n
c_{1i}\,
\sigma\!\left(\sum_{j=1}^r \xi_{1ij}z_j-\theta_{1i}\right),
\ \dots,\
\sum_{i=1}^n
c_{pi}\,
\sigma\!\left(\sum_{j=1}^r \xi_{pij}z_j-\theta_{pi}\right)
\Bigg),
\]
and define $\mathcal{T}:\mathbb{R}^d\to\mathbb{R}^p$ by
\[
\mathcal{T}(y):=\bigl(\sigma(\omega_1\cdot y+\zeta_1),\dots,
\sigma(\omega_p\cdot y+\zeta_p)\bigr).
\]
Both maps are continuous because $\sigma$ is continuous and they are finite
linear combinations and compositions of continuous functions.

For $u\in V$ and $y\in K_2$, substituting
$z=(u(x_1),\dots,u(x_r))$ gives
\[
\big\langle \mathcal{B}(u(x_1),\dots,u(x_r)),\,\mathcal{T}(y)\big\rangle
=
\sum_{k=1}^p\sum_{i=1}^n
c_{ki}\,
\sigma\!\left(\sum_{j=1}^r \xi_{kij}u(x_j)-\theta_{ki}\right)
\sigma(\omega_k\cdot y+\zeta_k).
\]
Combining this identity with \eqref{eq:Lu-cor-prev} yields
\[
\sup_{u\in V}\sup_{y\in K_2}
\left|
G(u)(y)-
\big\langle
\mathcal{B}\big(u(x_1),\dots,u(x_r)\big),\,\mathcal{T}(y)
\big\rangle
\right|
<\varepsilon,
\]
as required.

The final statement follows by approximating the continuous maps
$\mathcal{B}$ and $\mathcal{T}$ uniformly on the relevant compact sets by
networks from any dense architecture class.
\end{proof}

\begin{remark}
Lanthaler et al.\ \cite{Lant} establish a universality result for
DeepONets in a probabilistic setting: the input space is equipped with a
probability measure and the approximation is formulated in an $L^2$ metric.
Within this framework they obtain approximation results for measurable
operators, thereby removing the continuity and compactness assumptions that
appear in the classical operator approximation theorem of Chen and Chen
\cite{Chen2}. Their work also provides quantitative error bounds and complexity
estimates for DeepONets, which lie outside the scope of the present
approximation-theoretic framework.

In comparison with the Chen--Chen theorem, the universality result of
Lanthaler et al.\ comes at the expense of considering a weaker distance,
namely $L^2$ instead of the uniform $L^\infty$ distance
(see \cite[Remark~3.1]{Lant}).

Thus, while \cite{Lant} relaxes continuity and compactness assumptions by
working in a probabilistic $L^2$ framework, the present work preserves uniform
approximation and extends operator approximation theory beyond Banach input
domains to general locally convex spaces. In particular, the classical
Chen--Chen theorem and the DeepONet approximation theorem of Lu et al.\ appear
as special cases of our framework.
\end{remark}

\section{Examples illustrating Theorem~\ref{thm:operator-separable}}

In this section we illustrate Theorem~\ref{thm:operator-separable} for several
important normed and locally convex spaces. In all examples the admissible
measurements are continuous linear functionals from the dual space $X^*$.
By Theorem~\ref{thm:operator-separable}, for a compact set $V\subset X$ and a continuous operator
\[
G:V\to C(K;\mathbb{R}^m),
\qquad K\subset\mathbb{R}^d\ \text{compact},
\]
we obtain approximations of the form
\begin{equation}\label{eq:example-template}
G(u)(y)\approx
\sum_{k=1}^N a_k(u)\,\sigma(\omega_k\cdot y+\zeta_k),
\qquad u\in V,\ y\in K,
\end{equation}
where each coefficient map $a_k:X\to\mathbb{R}^m$ is a topological neural network on $X$
constructed as in Definition~\ref{def:1layer}.

\medskip

\noindent\textbf{Example 1.}
Let $X=\mathbb{R}^d$ or, more generally, let
\[
X=M_{n\times p}(\mathbb{R})
\]
be the vector space of all $n\times p$ real matrices equipped with any norm.
Then $X$ is a finite-dimensional Banach space and every continuous linear
functional on $X$ can be written in the form
\[
A\longmapsto \mathrm{trace}(W^{T}A), \qquad A\in X,
\]
for some matrix $W\in M_{n\times p}(\mathbb{R})$.

Let $V\subset X$ be compact and let $G:V\to C(K;\mathbb{R}^m)$
be continuous. Applying Theorem~\ref{thm:operator-separable} yields approximations of the form
\eqref{eq:example-template}, where each coefficient map $a_k:X\to\mathbb{R}^m$
is a topological neural network on $X$ in the sense of Definition~\ref{def:1layer}.
More precisely, there exist matrices $W_{k,1},\dots,W_{k,r_k}\in M_{n\times p}(\mathbb{R})$,
vectors $c_{k,i}\in\mathbb{R}^m$, and scalars
$\theta_{k,i}\in\mathbb{R}$ such that
\[
a_k(A)
=
\sum_{i=1}^{r_k}
c_{k,i}\,
\sigma\!\bigl(\mathrm{trace}(W_{k,i}^T A)-\theta_{k,i}\bigr),
\qquad A\in X.
\]

Thus the branch part of the network uses finitely many linear measurements
of the input matrix $A$ of the form $A\mapsto \mathrm{trace}(W^T A)$,
while the trunk part consists of ridge functions on $K$.

\medskip

\noindent\textbf{Example 2.}
Let $1\le p<\infty$ and let $X=\ell_p$, the Banach space of real sequences
$x=(x_1,x_2,\dots)$ with
\[
\|x\|_p=\left(\sum_{n=1}^\infty |x_n|^p\right)^{1/p}<\infty .
\]
Then the continuous dual is $X^*=\ell_q$, $q=\frac{p}{p-1}$, and every
continuous linear functional on $\ell_p$ can be written in the form
\[
x\longmapsto \sum_{n=1}^\infty w_n x_n,\qquad
w=(w_1,w_2,\dots)\in\ell_q .
\]

Let $V\subset\ell_p$ be compact and let
$G:V\to C(K;\mathbb{R}^m)$ be continuous.
Applying Theorem~\ref{thm:operator-separable} yields approximations of the
form \eqref{eq:example-template}, where each coefficient map
$a_k:X\to\mathbb{R}^m$ is a topological neural network on $X$ in the sense of
Definition~\ref{def:1layer}. More precisely, there exist vectors
\[
w^{(k,1)},\dots,w^{(k,r_k)}\in\ell_q,
\]
vectors $c_{k,i}\in\mathbb{R}^m$, and scalars $\theta_{k,i}\in\mathbb{R}$ such that
\[
a_k(x)
=
\sum_{i=1}^{r_k}
c_{k,i}\,
\sigma\!\left(
\sum_{n=1}^\infty w^{(k,i)}_n x_n-\theta_{k,i}
\right),
\qquad x\in X.
\]

\medskip

\noindent\textbf{Example 3.}
Let $X=c_0$, the Banach space of real sequences converging to zero, equipped
with the sup norm. Then the continuous dual is $X^*=\ell_1$, and every
continuous linear functional on $c_0$ can be written in the form
\[
x\longmapsto \sum_{n=1}^\infty w_n x_n,\qquad
w=(w_1,w_2,\dots)\in\ell_1 .
\]

Let $V\subset c_0$ be compact and let $G:V\to C(K;\mathbb{R}^m)$ be continuous.
Applying Theorem~\ref{thm:operator-separable} yields approximations of the
form \eqref{eq:example-template}, where each coefficient map
$a_k:X\to\mathbb{R}^m$ is a topological neural network on $X$ in the sense of
Definition~\ref{def:1layer}. More precisely, there exist vectors
\[
w^{(k,1)},\dots,w^{(k,r_k)}\in\ell_1,
\]
vectors $c_{k,i}\in\mathbb{R}^m$, and scalars $\theta_{k,i}\in\mathbb{R}$ such that
\[
a_k(x)
=
\sum_{i=1}^{r_k}
c_{k,i}\,
\sigma\!\left(
\sum_{n=1}^\infty w^{(k,i)}_n x_n-\theta_{k,i}
\right),
\qquad x\in X.
\]

\medskip

\noindent\textbf{Example 4.}
Let $(\Omega,\mu)$ be a measure space and let $X=L_p(\Omega,\mu)$,
$1\le p<\infty$. If $p=1$, assume in addition that $(\Omega,\mu)$ is
$\sigma$-finite. Then $X$ is Banach and its continuous dual is
$X^*=L_q(\Omega,\mu)$, where $q=\frac{p}{p-1}$ (with $q=\infty$ if $p=1$),
and every continuous linear functional on $X$ has the form
\[
f\longmapsto \int_\Omega f(x)g(x)\,d\mu(x),\qquad g\in L_q(\Omega,\mu).
\]

Let $V\subset L_p(\Omega,\mu)$ be compact and let
$G:V\to C(K;\mathbb{R}^m)$ be continuous.
Applying Theorem~\ref{thm:operator-separable} yields approximations of the
form \eqref{eq:example-template}, where each coefficient map
$a_k:X\to\mathbb{R}^m$ depends on finitely many integral measurements.
More precisely, there exist functions
$g_{k,1},\dots,g_{k,r_k}\in L_q(\Omega,\mu)$, vectors
$c_{k,i}\in\mathbb{R}^m$, and scalars $\theta_{k,i}\in\mathbb{R}$ such that
\[
a_k(f)
=
\sum_{i=1}^{r_k}
c_{k,i}\,
\sigma\!\left(\int_\Omega f(x)g_{k,i}(x)\,d\mu(x)-\theta_{k,i}\right),
\qquad f\in X.
\]

\medskip

\noindent\textbf{Example 5.}
Let $X=\mathcal{S}(\mathbb{R}^n)$, the Schwartz space of rapidly decreasing
smooth functions. This is a Fréchet space whose continuous dual
$X^*=\mathcal{S}'(\mathbb{R}^n)$ consists of tempered distributions.
Every continuous linear functional has the form
\[
f\longmapsto \langle T,f\rangle,\qquad T\in\mathcal{S}'(\mathbb{R}^n).
\]

Let $V\subset\mathcal{S}(\mathbb{R}^n)$ be compact and let
$G:V\to C(K;\mathbb{R}^m)$ be continuous.
Applying Theorem~\ref{thm:operator-separable} yields approximations of the
form \eqref{eq:example-template}, where each coefficient map
$a_k:X\to\mathbb{R}^m$ depends on finitely many distributional measurements.
More precisely, there exist tempered distributions
$T_{k,1},\dots,T_{k,r_k}\in\mathcal{S}'(\mathbb{R}^n)$, vectors
$c_{k,i}\in\mathbb{R}^m$, and scalars $\theta_{k,i}\in\mathbb{R}$ such that
\[
a_k(f)
=
\sum_{i=1}^{r_k}
c_{k,i}\,
\sigma\!\left(\langle T_{k,i},f\rangle-\theta_{k,i}\right),
\qquad f\in X.
\]

\medskip

\noindent\textbf{Example 6.}
Let $X=\mathcal{D}(U)$, the space of smooth compactly supported functions on
an open set $U\subset\mathbb{R}^n$. Its continuous dual
$X^*=\mathcal{D}'(U)$ is the space of distributions, and every continuous
linear functional has the form
\[
f\longmapsto \langle T,f\rangle,\qquad T\in\mathcal{D}'(U).
\]

Let $V\subset\mathcal{D}(U)$ be compact and let
$G:V\to C(K;\mathbb{R}^m)$ be continuous.
Applying Theorem~\ref{thm:operator-separable} yields approximations of the
form \eqref{eq:example-template}, where each coefficient map
$a_k:X\to\mathbb{R}^m$ depends on finitely many distributional measurements.
More precisely, there exist distributions
$T_{k,1},\dots,T_{k,r_k}\in\mathcal{D}'(U)$, vectors
$c_{k,i}\in\mathbb{R}^m$, and scalars $\theta_{k,i}\in\mathbb{R}$ such that
\[
a_k(f)
=
\sum_{i=1}^{r_k}
c_{k,i}\,
\sigma\!\left(\langle T_{k,i},f\rangle-\theta_{k,i}\right),
\qquad f\in X.
\]

\medskip

These examples illustrate that Theorem~\ref{thm:operator-separable} provides a unified
approximation principle for continuous operators on compact subsets of
finite-dimensional spaces, classical Banach spaces of sequences and functions,
and important non-normable locally convex spaces arising in analysis.

\section{Conclusion}

We present a topological extension of DeepONets in which the operator input lies
in an arbitrary locally convex topological vector space and
the network architecture is constructed using continuous linear
functionals from the dual space. Under the assumption that the activation
function is a Tauber--Wiener function, we prove a universal approximation theorem
showing that continuous nonlinear operators can be approximated on compact subsets of the
input space by finite separable expansions.

The classical Chen--Chen operator approximation theorem and the dot-product
DeepONet approximation theorem of Lu et al.\ arise as special cases of our
results. Several examples illustrate that the main theorem applies to a wide
range of spaces, including infinite-dimensional Banach spaces and non-normable
locally convex spaces.

\end{document}